\def\year{2016}\relax
\documentclass[letterpaper]{article}
\usepackage{aaai16}
\usepackage{times}
\usepackage{helvet}
\usepackage{courier}
\usepackage{amsthm, amsmath, amsfonts}
\usepackage{booktabs}
\usepackage{footnote}
\usepackage[pdftex]{graphicx}

\begin{document}
\title{Adaptive Normalized Risk-Averting Training For Deep Neural Networks}
\author{Zhiguang Wang, Tim Oates \\ 
	Department of Computer Science and Electric Engineering \\
	University of Maryland Baltimore County \\
	zgwang813@gmail.com, oates@umbc.edu \\
	\And James Lo \\  
	Department of Mathematics and Statistics \\
	University of Maryland Baltimore County\\
	jameslo@umbc.edu
}
\maketitle
\begin{abstract}
	\begin{quote}
		This paper proposes a set of new error criteria and a learning approach, called Adaptive Normalized Risk-Averting Training (ANRAT) to attack the non-convex optimization problem in training deep neural networks without pretraining. Theoretically, we demonstrate its effectiveness based on the expansion of the convexity region. By analyzing the gradient on the convexity index $\lambda$, we explain the reason why our learning method using gradient descent works. In practice, we show how this training method is successfully applied for improved training of deep neural networks to solve visual recognition tasks on the MNIST and CIFAR-10 datasets. Using simple experimental settings without pretraining and other tricks, we obtain results comparable or superior to those reported in recent literature on the same tasks using standard ConvNets + MSE/cross entropy. Performance on deep/shallow multilayer perceptron and Denoised Auto-encoder is also explored. ANRAT can be combined with other quasi-Newton training methods, innovative network variants, regularization techniques and other common tricks in DNNs. Other than unsupervised pretraining, it provides a new perspective to address the non-convex optimization strategy in training DNNs.  
	\end{quote}
\end{abstract}

\section{Introduction}
Deep neural networks (DNNs) are attracting attention largely due to their impressive empirical performance in image and speech recognition tasks. 
While Convolutional Networks (ConvNets) are the de facto  state-of-the-art for visual recognition, Deep Belief Networks (DBN), Deep Boltzmann Machines (DBM) and Stacked Auto-encoders (SA) provide insights as generative models to learn the full generating distribution of input data.  
Recently, researchers have investigated various techniques to improve the learning capacity of DNNs. Unsupervised pretraining using Restrict Boltzmann Machines (RBM), Denoised Autoencoders (DA) or Topographic ICA (TICA) has proved to be helpful for training DNNs with better weight initialization \cite{ngiam2010tiled,coates2011selecting}. 
Rectified Linear Unit (ReLU) and variants are proposed as the optimal activation functions to better interpret hidden features 
Various regularization techniques such as dropout \cite{srivastava2014dropout} with Maxout \cite{goodfellow2013maxout} are proposed to regulate the DNNs to be less prone to overfitting.  

Neural network models always lead to a non-convex optimization problem. The optimization algorithm impacts the quality of the local minimum because it is hard to find a global minimum or estimate how far a particular local minimum is from the best possible solution. The most standard approach to optimize DNNs is Stochastic Gradient Descent (SGD). There are many variants of SGD and researchers and practitioners typically choose a particular variant empirically. While nearly all DNNs optimization algorithms in popular use are gradient-based, recent work has shown that more advanced second-order methods such as L-BFGS and Saddle-Free Newton (SFN) approaches can yield better results for DNN tasks \cite{ngiam2011optimization,dauphin2014identifying}. Second order derivatives can be addressed by hardware extensions (GPUs or clusters) or batch methods when dealing with massive data, SGD still provides a robust default choice for optimizing DNNs.

Instead of modifying the network structure or optimization techniques for DNNs, we focused on designing a new error function to convexify the error space. The convexification approach has been studied in the optimization community for decades, but has never been seriously applied within deep learning. Two well-known methods are the graduated nonconvexity method \cite{blake1987visual} and the Liu–Floudas convexification method \cite{liu1993remark}. Liu–Floudas convexification can be applied to optimization problems where the error criterion is twice continuously differentiable, although determining the weight $\alpha$ of the added quadratic function for convexifying the error criterion involves significant computation when dealing with massive data and parameters. 

Following the same name employed for deriving robust controllers and filters \cite{speyer1974optimization}, a new type of Risk-Averting Error (RAE) is proposed theoretically for solving non-convex optimization problems \cite{lo2010convexification}. Empirically, with the proposal of Normalized Risk-Averting Error (NRAE) and the Gradual Deconvexification method (GDC), this error criterion is proved to be competitive with the standard mean square error (MSE) in single layer and two-layer neural networks for solving data fitting and classification problems \cite{gui2014pairwise,lo2012overcoming}. Interestingly, SimNets, a generalization of ConvNets that was recently proposed in \cite{cohen2014simnets}, uses the MEX operator (whose name stands for Maximum-minimum-Expectation Collapsing Smooth) as an activation function to generalize ReLU activation and max pooling. We notice that the MEX operator with $L_2$ units has exactly the \textit{same} mathematical form with NRAE. However, NRAE is still hard to optimize in practice due to plateaus and the unstable error space caused by the fixed large convexity index. GDC alleviates these problems but its performance is limited and suffers from the slow learning speed. Instead of fixing the convexity index $\lambda$, Adaptive Normalized Risk-Averting Training (ANRAT) optimizes NRAE by tuning $\lambda$ adaptively using gradient descent. We give theoretical proofs of its optimal properties against the standard $L_p$-norm error. Our experiments on MNIST and CIFAR-10 with different deep/shallow neural nets demonstrate the effectiveness empirically. Being an optimization algorithm, our approach are not supposed to deal specifically with the problem of over-fitting, however we show that this can be handled by the usual methods of regularization such as weight decay or dropout.

\section{Convexification on Error Criterion}
We begin with the definition of RAE for the $L_p$ norm and the theoretical justifications on its convexity property. RAE is not suitable for real applications since it is not bounded. Instead, NRAE is bounded to overcome the register overflow in real implementations. We prove that NRAE is quasi-convex, and thus shares the same global and local optimum with RAE. Moreover, we show the lower-bound of its performance is as good as $L_p$-norm error when the convexity index satisfies a constraint, which theoretically supports the ANRAT method proposed in the next section.     

\subsection{Risk-averting Error Criterion}
\newtheorem{mythm}{Theorem}
\newtheorem{mydef}{Definition}

Given training samples $\{\boldsymbol{X},y\} = \{(\boldsymbol{x}_1, y_1), (\boldsymbol{x}_2, y_2), ..., (\boldsymbol{x}_m, y_m)\}$, the function $f(\boldsymbol{x}_i, \boldsymbol{W})$ is the learning model with parameters $W$. The loss function of $L_p$-norm error is defined as: 
\begin{eqnarray}
	l_p(f(\boldsymbol{x}_i,\boldsymbol{W}), y_i) = \frac{1}{m}\sum_{i=1}^{m}||f(\boldsymbol{x}_i,\boldsymbol{W})-y_i||^p
	\label{eqn:lp-error}
\end{eqnarray}

When $p=2$, Eqn. \ref{eqn:lp-error} denotes to the standard Mean Square Error (MSE). The Risk-Averting Error criterion (RAE) corresponding to the $L_p$-norm error is defined by
\begin{eqnarray}
	RAE_{p,q}(f(\boldsymbol{x}_i,\boldsymbol{W}), y_i) = \frac{1}{m}\sum_{i=1}^{m}e^{\lambda^q||f(\boldsymbol{x}_i,\boldsymbol{W})-y_i||^p}
	\label{eqn:lp-RAE}
\end{eqnarray}

$\lambda$ is the convexity index. It controls the size of the convexity region. 

%
%

Because RAE has the sum-exponential form, its Hessian matrix is tuned exactly by the convexity index $\lambda^q$. The following theorem indicates the relation between the convexity index and its convexity region.
\begin{mythm}[Convexity]
	Given the Risk-Averting Error criterion $RAE_{p,q}$ ($p, q \in \mathcal{N^+}$), which is twice continuous differentiable. $J_{p,q}(\boldsymbol{W})$ and $H_{p,q}(\boldsymbol{W})$ are the corresponding Jacobian and Hessian matrix. As $\lambda \to \pm\infty$, the convexity region monotonically expands to the entire parameter space except for the subregion $S:=\{W \in \mathcal{R}^n | rank(H_{p,q}(\boldsymbol{W}))<n, H_{p,q}(\boldsymbol{W}<0)\}$. 
	\label{thm:RAE} 
\end{mythm}

Please refer to the supplementary material for the proof. Intuitively, the use of the RAE was motivated by its emphasizing large individual deviations in approximating functions and optimizing parameters in an exponential manner, thereby avoiding such large individual deviations and achieving robust performances. Theoretically, Theorem \ref{thm:RAE} states that when the convexity index $\lambda$ increases to infinity, the convexity region in the parameter space of RAE expands monotonically to the entire space except the intersection of a finite number of lower dimensional sets. The number of sets increases rapidly as the number $m$ of training samples increases. Roughly speaking, larger $\lambda$ and $m$ cause the size of the convexity region to grow larger respectively in the error space of RAE.

When $\lambda \to \infty$, the error space can be perfectly stretched to be strictly convex, thus avoid the local optimum to guarantee a global optimum. Although RAE works well in theory, it is not bounded and suffers from the exponential magnitude and arithmetic overflow when using gradient descent in implementations .     

\subsection{Normalized Risk-Averting Error Criterion}
RAE ensures the convexity of the error space to find the global optimum. By using NRAE, we relax the global optimum problem by finding a better local optimum to meet a theoretically and practically reasonable trade-off in real applications.    

Given training samples $\{\boldsymbol{X},y\} = \{(\boldsymbol{x}_1, y_1), (\boldsymbol{x}_2, y_2), ..., (\boldsymbol{x}_m, y_m)\}$, the function $f(\boldsymbol{x}_i, \boldsymbol{W})$ is the learning model with parameters $W$. The Normalized Risk-Averting Error Criterion (NRAE) corresponding to the $L_p$-norm error is defined as:

\begin{eqnarray} 
	&& NRAE_{p,q}(f(\boldsymbol{x}_i,\boldsymbol{W}), y_i) \nonumber \\
	&& = \frac{1}{\lambda^q}\log RAE_{p,q}(f(\boldsymbol{x}_i,\boldsymbol{W}), y_i) \nonumber \\
	&& = \frac{1}{\lambda^q}\log \frac{1}{m}\sum_{i=1}^{m}e^{\lambda^q||f(\boldsymbol{x}_i,\boldsymbol{W})-y_i||^p}
	\label{eqn:lp-NRAE}
\end{eqnarray}

\begin{mythm}[Bounded]
	$NRAE_{p,q}(f(\boldsymbol{x}_i,\boldsymbol{W}), y_i)$ is bounded. 
\end{mythm}


The proof is provided in the supplemental materials. Briefly, NRAE is bounded by functions independent of $\lambda$ and no overflow occurs for $\lambda \gg 1$. The following theorem states the quasi-convexity of NRAE. 

\begin{mythm}[Quasi-convexity]
	Given a parameter space $\{W \in \mathcal{R}^n \}$, Assume $\exists$ $\psi(W)$, s.t. $H_{p,q}(\boldsymbol{W}) > 0$ when $|\lambda^q|> \psi(W)$ to guarantee the convexity of $RAE_{p,q}(f(\boldsymbol{x}_i,\boldsymbol{W}), y_i)$. Then, $NRAE_{p,q}(f(\boldsymbol{x}_i,\boldsymbol{W}), y_i)$ is  \textbf{quasi-convex} and share the same local and global optimum with $RAE_{p,q}(f(\boldsymbol{x}_i,\boldsymbol{W}), y_i)$.
	\label{thm:quasi}
\end{mythm}

\begin{proof}
	If $RAE_{p,q}(f(\boldsymbol{x}_i,\boldsymbol{W}), y_i)$ is convex, it is quasi-convex. $\log$ function is monotonically increasing, so the composition  $\log RAE_{p,q}(f(\boldsymbol{x}_i,\boldsymbol{W}), y_i)$ is quasi-convex. \footnote{Because the function $f$ defined by $f(x) = g(U(x))$ is quasi-convex if the function $U$ is quasiconvex and the function $g$ is increasing.}  
	
	$\log$ is a strictly monotone function and $NRAE_{p,q}(f(\boldsymbol{x}_i,\boldsymbol{W}), y_i)$ is quasi-convex, so it shares the same local and global minimizer with $RAE_{p,q}(f(\boldsymbol{x}_i,\boldsymbol{W}), y_i)$.
\end{proof}

The convexity region of NRAE is consistent with RAE. To interpret this statement in another perspective, the $\log$ function is a strictly monotone function. Even  if RAE is not strictly convex, NRAE still shares the same local and global optimum with RAE. If we define the mapping function $f:RAE \to NRAE$, it is easy to see that $f$ is  bijective and continuous. Its inverse map $f^{-1}$ is also continuous, so that $f$ is an open mapping. Thus, it is easy to prove that the mapping function $f$ is a homeomorphism to preserve all the topological properties of the given space.

The above theorems state the consistent relations among NRAE, RAE and MSE. It is proven that the greater the convexity index $\lambda$, the larger is the convex region is. Intuitively, increasing $\lambda$ creates tunnels for a local-search minimization procedure to travel through to a good local optimum.  However, we care about the justification on the advantage of NRAE against MSE.  Theorem \ref{thm:lowerbound} provides the theoretical justification for the performance lower-bound of NRAE.      

\begin{mythm}[Lower-bound]
	Given training samples $\{\boldsymbol{X},y\} = \{(\boldsymbol{x}_1, y_1), (\boldsymbol{x}_2, y_2), ..., (\boldsymbol{x}_m, y_m)\}$ and the model $f(\boldsymbol{x}_i, \boldsymbol{W})$ with parameters $W$. If $\lambda^q \geq 1$, $p, q \in \mathcal{N^+}$ and $p \geq 2$, then both $RAE_{p,q}(f(\boldsymbol{x}_i,\boldsymbol{W}), y_i)$ and  $NRAE_{p,q}(f(\boldsymbol{x}_i,\boldsymbol{W}), y_i)$ always have the higher chance to find a better local optimum than the standard $L_p$-norm error due to the expansion of the convexity region.
	\label{thm:lowerbound}
\end{mythm}

\begin{proof}
	Let $h_p(\boldsymbol{W})$ denotes the Hessian matrix of standard $L_p$-norm error (Eqn. \ref{eqn:lp-error}), note $\alpha_i(\boldsymbol{W})=f(\boldsymbol{x}_i,\boldsymbol{W})-y_i$ we have
	\begin{eqnarray}
		h_p(\boldsymbol{W}) &=&\frac{p}{m}\sum_{i=1}^{m}\{(p-1)\alpha_i(\boldsymbol{W})^{p-2} \frac{\partial f(\boldsymbol{x}_i,\boldsymbol{W})^2}{\partial \boldsymbol{W}}  \nonumber \\ &+& \alpha_i(\boldsymbol{W})^{p-1}\frac{\partial f^2(\boldsymbol{x}_i,\boldsymbol{W})}{\partial \boldsymbol{W}^2}\}
	\end{eqnarray}
	Since $\lambda^q \geq 1$, let $diag_{eig}$ denotes the diagonal matrix of the eigenvalues from SVD decomposition. $\succeq$ here means 'element-wise greater'. When $A \succeq B$, each element in A is greater than B. Then we have
	\begin{align}
		diag_{eig}[H_{p,q}(\boldsymbol{W})]  &\succeq diag_{eig}[h_p(\boldsymbol{W})
		+  \nonumber \\ &\frac{p^2}{m} \sum_{i=1}^{m}||\alpha_i(\boldsymbol{W})||^{2p-2}\frac{\partial f(\boldsymbol{x}_i,\boldsymbol{W})}{\partial \boldsymbol{W}}^2\}] \nonumber \\
		&\succeq diag_{eig}[h_{p}(\boldsymbol{W})]
	\end{align}	
	This indicates that the $RAE_{p,q}(f(\boldsymbol{x}_i,\boldsymbol{W}), y_i)$ always has larger convexity regions than the standard $L_p$-norm error to better enable escape of local minima. Because $NRAE_{p,q}(f(\boldsymbol{x}_i,\boldsymbol{W}), y_i)$ is quasi-convex, sharing the same local and global optimum with $RAE_{p,q}(f(\boldsymbol{x}_i,\boldsymbol{W}), y_i)$, the above conclusions are still valid.  
\end{proof}

Roughly speaking, NRAE always has a larger convexity region than the standard $L_p$-norm error in terms of their Hessian matrix when $\lambda \geq 1$. This property guarantees the higher probability to escape poor local optima using NRAE. In the worst case, NRAE will perform as good as standard $L_p$-norm error if the convexity region shrinks as $\lambda$ decreases or the local search deviates from the "tunnel" of convex regions.   

More specifically, $NRAE_{p,q}(f(\boldsymbol{x}_i,\boldsymbol{W}), y_i)$
\begin{itemize}
	\item approaches the standard $L_p$-norm error as $\lambda^q \to 0$.
	\item approaches the minimax error criterion $Min \alpha_{max}(\boldsymbol{W})$ as $\lambda^q \to \infty$.
\end{itemize} 


Please refer to the supplemental materials for the proofs. More rigid proofs that can be generalized to $L_p$-norm error are also given in \cite{lo2010convexification}. In SimNets, the authors also include quite similar discussions about the robustness with respect to $L_p$-norm error \cite{cohen2014simnets}.  

\section{Learning Methods}
We propose a novel learning method to training DNNs with NRAE, called the Adaptive Normalized Risk-Avering Training (ANRAT) approach. Instead of manually tuning $\lambda$ like GDC \cite{lo2012overcoming}, we  learn $\lambda$ adaptively in error backpropagation by considering $\lambda$ as a parameter instead of a hyperparameter. The learning procedure is standard batch SGD. We show it works quite well in theory and practice. 

The loss function of ANRAT is

\begin{eqnarray}
	l(W,\lambda) = \frac{1}{\lambda^q}\log \frac{1}{m}\sum_{i=1}^{m}e^{\lambda^q||f(\boldsymbol{x}_i,\boldsymbol{W})-y_i||^p} + a ||\lambda||^{-r}
	\label{eqn:genralloss}
\end{eqnarray}   

Together with NRAE, we also use a penalty term $a||\lambda||^{-r}$ to control the changing rate of $\lambda$. While minimize the NRAE score, small $\lambda$ is penalized to regulate the convexity region. $a$ is a hyperparameter to control the penalty index. The first-order derivatives on weight and $\lambda$ are

\begin{eqnarray}
	\frac{dl(W, \lambda)}{dW}  &=& \frac{p\sum_{i=1}^{m}e^{\lambda^q \alpha_i(W)^{p-1}}\frac{\partial f(\boldsymbol{x}_i,\boldsymbol{W})}{\partial \boldsymbol{W}}}{\sum_{i=1}^{m}e^{\lambda^q \alpha_i(W)^{p-1}}} \\
	\frac{dl(W, \lambda)}{d\lambda} &=& \frac{-q}{\lambda^{q+1}} \log \frac{1}{m} \sum_{i=1}^{m}e^{\lambda^q\alpha_i(W)^p} \\ &+& \frac{q}{\lambda}\frac{\sum_{i=1}^{m}e^{\lambda^q \alpha_i(W)^{p}}\alpha_i(W)^p}{\sum_{i=1}^{m}e^{\lambda^q \alpha_i(W)^{p}}}  \\ &-& ar\lambda^{-r-1} \label{eqn:glambda}
\end{eqnarray}

We make a transformation on Eqn. \ref{eqn:glambda} to better understand the gradient with respect to $\lambda$. Note that $k_i = \frac{e^{\lambda^q \alpha_i(W)^{p}}}{\sum_{i=1}^{m}e^{\lambda^q \alpha_i(W)^{p}}}$ is actually performing like a probability ($\sum_{i=1}^{m}k_i=1$).  Ignoring the penalty term, Eqn. \ref{eqn:glambda} can be formulated as follows:
\begin{eqnarray}
	\frac{dl(W, \lambda)}{d\lambda} &=&
	\frac{q}{\lambda}(\sum_{i=1}^{m}k_i\alpha_i(\boldsymbol{W})^p-NRAE) \nonumber \\
	&=& \frac{q}{\lambda}(\mathbb{E}(\alpha(\boldsymbol{W})^p)-NRAE) \nonumber \\
	&\approx&  \frac{q}{\lambda}(L_P\text{-norm error}-NRAE)
	\label{eqn:nglambda}
\end{eqnarray} 

Note that as $\alpha_i(\boldsymbol{W})^p$ becomes smaller, the expectation on $\alpha_i(\boldsymbol{W})^p$ approaches the standard $L_p$-norm error. Thus, the gradient on $\lambda$ is approximately the difference between NRAE and the standard $L_p$-norm error. Because large $\lambda$ can incur plateaus to prevent NRAE from finding better optima using batch SGD \cite{lo2012overcoming}, they need GDC to gradually deconvexify the NRAE to make the error space well shaped and stable. Through Eqn. \ref{eqn:nglambda}, ANRAT solve this problem in a more flexible and adaptive manner. When NRAE is larger, Eqn. \ref{eqn:nglambda} remains negative and makes $\lambda$ increase to enlarge the convexity region, facilitating the search in the error space for better optima. When NRAE is smaller, the learned parameters are seemingly going through the optimal "tunnel" for better optima. Eqn. \ref{eqn:nglambda} becomes positive to decrease $\lambda$ and helps NRAE not deviate far from the manifold of the standard $L_p$-norm error to make the error space stable without large plateaus. Thus, ANRAT adaptively adjusts the convexity index to find an optimal trade-off between better solutions and stability.      

This training approach has more flexibility. The gradient on $\lambda$ as the weighted difference between NRAE and the standard $L_P$-norm error, enables NRAE to approach the $L_P$-norm error by adjusting $\lambda$ gradually. Intuitively, it keeps searching the error space near the manifold of the $L_p$-norm error to find better optima in a way of competing with and at the same time relying on the standard $L_p$-norm error space.  

In Eqn. \ref{eqn:genralloss}, the penalty weight $a$ and index $r$ control the convergence speed by penalizing small $\lambda$. Smaller $a$ emphasizes tuning $\lambda$ to allow faster convergence speed between NRAE and $L_p$-norm error. Larger $a$ forces larger $\lambda$ for a better chance to find a better local optimum but runs the risk of plateaus and deviating far from the stable error space. $r$ regulates the magnitude of $\lambda$ and its derivatives in gradient descent.         

\section{Experiments}
We present the results from a series of experiments designed on the MNIST and CIFAR-10 datasets to test the effectiveness of ANRAT for visual recognition with DNNs. We did not explore the full hyperparameters in Eqn. \ref{eqn:genralloss}. Instead we fix the hyperparameters at $p=2$, $q=2$ and $r=1$ to mainly compare with MSE. So the final loss function of ANRAT we optimized is

\begin{eqnarray}
	l(W,\lambda) = \frac{1}{\lambda^2}\log \frac{1}{m}\sum_{i=1}^{m}e^{\lambda^2||f(\boldsymbol{x}_i,\boldsymbol{W})-y_i||^2} + a |\lambda|^{-1}
	\label{eqn:fixloss}
\end{eqnarray}

This loss function is minimized by batch SGD without complex methods, such as momentum, adaptive/hand tuned learning rates or tangent prop. The learning rate and penalty weight $a$ are selected in $\{1, 0.5, 0.1\}$ and $\{1, 0.1, 0.001\}$ on validation sets respectively. The initial $\lambda$ is fixed at 10. We use the hold-out validation set to select the best model, which is used to make predictions on the test set. All experiments are implemented quite easily in Python and Theano to obtain GPU acceleration \cite{Bastien-Theano-2012}. 

The MNIST dataset \cite{lecun1998gradient} consists of hand written digits 0-9 which are 28x28 in size. There are 60,000 training images and 10,000 testing images in total. We use 10000 images in training set for validation to select the hyperparameters and report the performance on the test set. We test our method on this dataset without data augmentation.

The CIFAR-10 dataset  \cite{krizhevsky2009learning} is composed of 10 classes of natural images. There are 50,000 training images in total and 10,000 testing images. Each image is an RGB image of size 32x32. For this dataset, we adapt pylearn2 \cite{goodfellow2013pylearn2} to apply the same global contrast normalization and ZCA whitening as was used by Goodfellow et. al \cite{goodfellow2013maxout}. We use the last 10,000 images of the training set as validation data for hyperparameter selection and report the test accuracy.

\section{Results and Discussion}
\subsection{Results on ConvNets}
\begin{figure*}[t]
	\centering
	\includegraphics[width=0.8\textwidth]{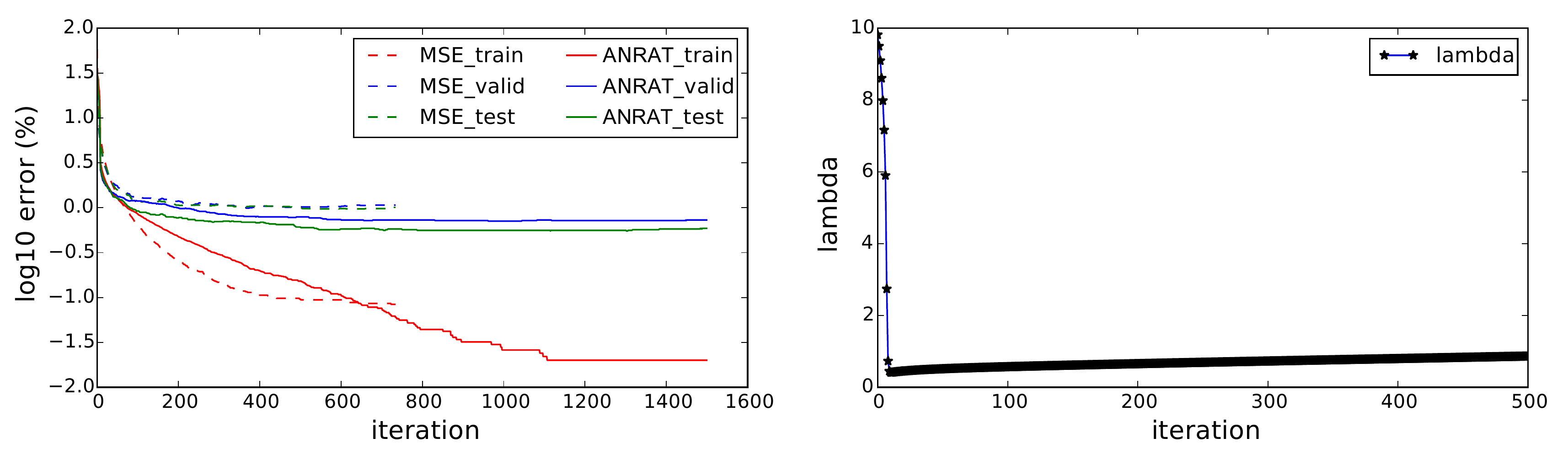}
	\caption{(a). MNIST train, validation and test error rates throughout training with Batch SGD for MSE and ANRAT with $l_2$ priors (left). (b). The curve of $\lambda$ throughout ANRAT training (right).} 
	\label{fig:MNISTcurve}
\end{figure*}

\begin{savenotes}
	\begin{table}[t]
		\centering
		\begin{tabular}{ll}
			\toprule
			Method\footnote{
				(1)\cite{mairal2014convolutional};(2)\cite{lee2014deeply};(3)\cite{lecun1998gradient};(4)\cite{ranzato2007unsupervised};(5)\cite{ngiam2011optimization};(6)\cite{ranzato2007unsupervised};(7)\cite{poultney2006efficient};(8)\cite{zeiler2013stochastic};(9)\cite{jarrett2009best}} & Error \% \\
			\midrule 
			Convolutional Kernel Networks + L-BFGS-B$^{(1)}$  & \textbf{0.39} \\
			Deeply Supervised Nets + dropout$^{(2)}$  & \textbf{0.39} \\
			\midrule
			\midrule
			ConvNets (Lenet-5)$^{(3)}$ & 0.95 \\
			ConvNets  + MSE/CE (this paper)  & 0.93 \\
			large ConvNets, random feature$^{(4)}$  & 0.89 \\
			ConvNets + L-BFGS$^{(5)}$  & 0.69 \\
			large ConvNets,  unsup pretraining$^{(6)}$  & 0.62  \\
			ConvNets, unsup pretraining$^{(7)}$  & 0.6  \\
			ConvNets + dropout$^{(8)}$   & 0.55  \\
			large ConvNets, unsup pretraining$^{(9)}$   & 0.53 \\
			ConvNets + ANRAT (This paper) & 0.52 \\
			ConvNets + ANRAT + dropout (This paper) & \textbf{0.39} \\
			\bottomrule
		\end{tabular}%
		\caption{Test set misclassification rates of the best methods that utilized convolutional networks on the original MNIST dataset using single model.}
		\label{tab:MNISTresults}%
	\end{table}%
\end{savenotes}

On the MNIST dataset we use the same structure of LeNet5 with two convolutional  max-pooling layers but followed by only one fully connected layer and a densely connected softmax layer. The first convolutional layer has 20 feature maps of size $5 \times 5$ and max-pooled by $2 \times 2$ non-overlapping windows. The second convolutional layer has 50 feature maps with the same convolutional and max-pooling size. The fully connected layer has 500 hidden units. An $l_2$ prior was used with the strength $0.05$ in the Softmax layer. Trained by ANRAT, we can obtain a test set error of 0.52\%, which is the best result we are aware of that does not use dropout on the pure ConvNets. We summarize the best published results on the standard MNIST dataset in Table \ref{tab:MNISTresults}.

The best performing neural networks for pure ConvNets that does not use dropout or unsupervised pretraining achieve an error of about 0.69\% \cite{ngiam2011optimization}. They demonstrated this performance with L-BFGS.  Using dropout, ReLU and a response normalization layer, the error reduces to 0.55\% \cite{zeiler2013stochastic}. Prior to that, Jarrett et. al showed by increasing the size of the network and using unsupervised pretraining, they can obtain a better result at 0.53\% \cite{jarrett2009best}. Previous state of the art is 0.39\% \cite{mairal2014convolutional,lee2014deeply} for a single model on the original MNIST dataset. Using batch SGD to optimize either CE or MSE on the ConvNets descried above, we can get an error rate at 0.93\%. Replacing the training methods with ANRAT using batch GD leads to a sharply decreased validation error of 0.66\% with a test error at 0.52\%. With dropout and ReLU the test error rate drops to 0.39\%, 
which is the same with the best results without averaging or data augmentation (Table \ref{tab:MNISTresults}) but we only use standard Convnets and simple experimental settings. 

Fig. \ref{fig:MNISTcurve} (a) shows the progression of training, validation and test errors over 160 training epochs. The errors trained on MSE plateau as it can not train the ConvNets sufficiently and seems like underfit. Using ANRAT, the validation and test errors remain decreasing along with the training error. During training, $\lambda$ sharply decrease, regulating the tunnel of NRAE to approach the manifold of MSE. Afterward the penalty term becomes significant, force $\lambda$ to grow gradually while expanding the convex region for higher chance to find the better optimum (Figure \ref{fig:MNISTcurve} (b)). 

%
%

Our next experiment is performed on the CIFAR-10 dataset. We observed significant overfitting using both MSE and ANRAT with the fixed learning rate and batch SGD, so dropout is applied to prevent the co-adaption of weights and improve generalization. We use a similar network layout as in \cite{srivastava2014dropout} but with only two convolutional  max-pooling layers. The first convolutional layer has 96 feature maps of size $5 \times 5$ and max-pooled by $2 \times 2$ non-overlapping windows. The second convolutional layer has 128 feature maps with the same convolutional and max-pooling size. The fully connected layer has 500 hidden units. Dropout was applied to all the layers of the network with the probability of retaining a hidden unit being $p = (0.9, 0.75, 0.5, 0.5, 0.5)$ for the different layers of the network.  Using batch SGD to optimize CE on the simple configuration of ConvNets + dropout, a test accuracy of 80.6 \% is achieved \cite{krizhevsky2012imagenet}. We also reported the performance at 80.58\% with MSE instead of CE with the similar network layout. Replacing the training methods with ANRAT using batch SGD gives a test accuracy of 85.15\%. This is superior to the results obtained by MSE/CE and unsupervised pretraining. 
In Table. \ref{tab:cifar10results}, our result with simple setting is shown to be competitive to those achieved by different ConvNet variants. 

%

\begin{savenotes}
	\begin{table}[t]
		\centering
		\caption{Test accuracy of the best methods that utilized convolutional framework on CIFAR-10 dataset without data augmentation.} 
		\begin{tabular}{ll}
			\toprule
			Method\footnote{
				(1)\cite{zeiler2013stochastic};(2)\cite{srivastava2014dropout};(3)\cite{goodfellow2013maxout};(4)\cite{zeiler2013stochastic};(5)\cite{coates2011selecting};(6)\cite{Lin2014Net};(7)\cite{lee2014deeply}} & Acc \% \\
			\midrule
			ConvNets + Stochastic pooling + dropout$^{(1)}$  & 84.87 \\
			ConvNets + dropout +Bayesian hyperopt$^{(2)}$   & 87.39 \\
			ConvNets + Maxout + dropout$^{(3)}$  & 88.32 \\
			Convolutional NIN + dropout$^{(6)}$ & 89.6 \\
			Deeply Supervised Nets + dropout$^{(7)}$ & \textbf{90.31} \\
			\midrule
			\midrule
			ConvNets + MSE + dropout  (this paper) & 80.58 \\
			ConvNets + CE + dropout$^{(4)}$  & 80.6 \\
			ConvNets + VQ unsup pretraining$^{(5)}$   & 82 \\    
			ConvNets + ANRAT + dropout (This paper) & \textbf{85.15} \\
			\bottomrule
		\end{tabular}%
		\label{tab:cifar10results}%
	\end{table}%
\end{savenotes}

\subsection{Results on Multilayer Perceptron}

On the MNIST dataset, MLPs with unsupervised pretraining has been  well studied in recent years, so we select this dataset to compare ANRAT in shallow and deep MLPs with MSE/CE and unsupervised pretraining. For the shallow MLPs, we follow the network layout as in \cite{gui2014pairwise,lecun1998gradient} that has only one hidden layer with 300 neurons. We build the stacked architecture and deep network using the same architecture as \cite{larochelle2009exploring} with 500, 500 and 2000 hidden units in the first, second and third layers, respectively. The training approach is purely batch SGD with no momentum or adaptive learning rate. No weight decay or other regularization technique is applied in our experiments.

Experiment results in Table. \ref{tab:MLPresults} show that the deep MLP classifier trained by the ANRAT method has the lowest test error rate (1.45\%) of benchmark MLP classifiers with MSE/CE under the same settings. It indicates that ANRAT has the ability to provide reasonable solutions with different initial weight vectors. This result is also better than deep MLP + supervised pretraining or Stacked Logistic Regression networks. We note that the deep MLP using unsupervised pretraining (auto-encoders or RBMs) remains to be the best with test error at 1.41\% and 1.2\%. Unsupervised pretraining is effective in initializing the weights to obtain a better local optimum. Compared with unsupervised pretraining + fine tuning, ANRAT sometimes still fall into the sightly worse local optima in this case. However, ANRAT is significantly better than MSE/CE without unsupervised pretraining.


Interestingly, we do not observe significant advantages with ANRAT in shallow MLPs. Although in early literature, the error rate on shallow MLPs were reported as 4.7\% \cite{lecun1998gradient} and 2.7\% with GDC \cite{gui2014pairwise}, both recent papers using CE \cite{larochelle2009exploring} and our own experiments with MSE can achieve error rate of 1.93\% and 2.02\%, respectively. Trained by ANRAT, we can have a test rate at 1.94\%. This performance is slightly better than MSE, but it is statistically identical to the performance obtained by CE. \footnote{in \cite{larochelle2009exploring}, the author do not report their network settings of the shallow MLP + CE, which may differ from 784-300-10.}. One possible reason is that in shallow networks which can be trained quite well by standard back propagation with normalized initializations, the local optimum achieved with MSE/CE is quite nearly a global optimum or good saddle point. Our result is also corresponding to the conclusion in \cite{dauphin2014identifying}, in which Dauphin et al. extend previous findings on networks with a single hidden layer to show theoretically and empirically that most badly suboptimal critical points are saddle points. Even with better convexity property, ANRAT is as good as MSE/CE in shallow MLPs. However, we find that the problem of poor local optimum becomes more manifest in deep networks. It is easier for ANRAT to find a way towards the better optimum near the manifold of MSE. For the sake of space, please refer to supplemental materials for the results on the shallow Denoised Auto-encoder. The conclusion is consistent that ANRAT performs better when attacking more difficult learning/fitting problems. While ANRAT is slightly better than CE/MSE + SGD on DA with uniform masking noise, it achieves a  significant performance boost when Gaussian block masking noise is applied.
\begin{savenotes}
	\begin{table}[t]
		\centering
		\caption{Test error rate of deep/shallow MLP with different training techniques.} 
		\begin{tabular}{ll}
			\toprule
			Method\footnote{(1)\cite{larochelle2009exploring};(2)\cite{lecun1998gradient};(3)\cite{gui2014pairwise}} & Error \% \\
			\midrule
			Deep MLP + supervised pretraining$^{(1)}$  & 2.04 \\
			Stakced Logistic Regression Network$^{(1)}$  & 1.85 \\
			Stacked Auto-encoder Network$^{(1)}$  & 1.41 \\
			Stacked RBM Network$^{(1)}$  & \textbf{1.2} \\
			\midrule
			\midrule
			Shallow MLP + MSE$^{(2)}$ & 4.7 \\
			Shallow MLP + GDC$^{(3)}$  & 2.7 $\pm$ 0.03 \\
			Shallow MLP + MSE (this paper) & 2.02 \\
			Shallow MLP + ANRAT (this paper) & 1.94 \\
			Shallow MLP + CE$^{(1)}$  & \textbf{1.93} \\
			\midrule
			Deep MLP + CE$^{(1)}$  & 2.4 \\
			Deep MLP + MSE (this paper) & 1.91 \\ 
			Deep MLP + ANRAT (this paper) & \textbf{1.45} \\    
			\bottomrule
		\end{tabular}%
		\label{tab:MLPresults}%
	\end{table}%
\end{savenotes}

\section{Conclusions and Outlook}
In this paper, we introduce a novel approach, Adaptive Normalized Risk-Averting Training (ANRAT), to help train deep neural networks. 
Theoretically, we prove the effectiveness of Normalized Risk-Averting Error on its arithmetic bound, global convexity and local convexity lower-bounded by standard $L_p$-norm error when convexity index $\lambda \geq 1$. By analyzing the gradient on $\lambda$, we explained the reason why using back propagation on $\lambda$ works. The experiments on deep/shallow network layouts demonstrate comparable or better performance with the same experimental settings among pure ConvNets and MLP + batch SGD on MSE and CE (with or without dropout). Other than unsupervised pretraining, it provides a new perspective to address the non-convex optimization strategy in DNNs. 


Finally, while these early results are very encouraging, clearly further research is warranted to address the questions that arise from non-convex optimization in deep neural networks. It is preliminarily showed that in order to generalize to a wide array of tasks, unsupervised and semi-supervised learning using unlabeled data is crucial. One interesting future work is to take advantage of unsupervised/semi-supervised pretraining with the non-convex optimization methods to train deep neural networks by finding the nearly global optimum. Another crucial question is to guarantee the generalization capability by preventing overfitting. Finally, we are quite interested in generalizing our approach to recurrent neural networks. We leave as future work any performance improvement on benchmark datasets by considering the cutting-edge approach to improve training and generalization performance.

\newpage
\bibliographystyle{aaai}
\bibliography{wang}
\end{document}